\documentclass[conference]{IEEEtran}
\IEEEoverridecommandlockouts
\usepackage{cite}
\usepackage{amsmath,amssymb,amsfonts}
\usepackage{amsthm}
\usepackage{bm}
\usepackage{algorithmic}
\usepackage{graphicx}
\usepackage{textcomp}
\usepackage{xcolor}
\newtheorem{proposition}{Proposition}
\newcommand{\R}{\mathbb{R}}
\newcommand{\E}{\mathbb{E}}
\newcommand{\vt}{v_\theta}
\newcommand{\x}{\bm{x}}
\def\BibTeX{{\rm B\kern-.05em{\sc i\kern-.025em b}\kern-.08em
    T\kern-.1667em\lower.7ex\hbox{E}\kern-.125emX}}

\begin{document}

\title{Difficulty-Calibrated Interpolation Paths for Conditional Flow Matching}

\author{\IEEEauthorblockN{Airin Akter Tania}
\IEEEauthorblockA{\textit{Dept. of Electrical and Electronic Engineering} \\
\textit{Khulna University of Engineering \& Technology}\\ Khulna, Bangladesh \\ airinaktertania18@gmail.com} 
\and
\IEEEauthorblockN{Md Raihan Khan}
\IEEEauthorblockA{\textit{Dept. of Electrical and Electronic Engineering} \\
\textit{North Western University}\\ Khulna, Bangladesh \\ kraihan918@gmail.com}
}

\maketitle

\begin{abstract}
Conditional Flow Matching trains generative models by regressing a
network onto the velocity of a prescribed noise-to-data interpolation
path. The interpolation schedule that shapes this path is known to
affect convergence and sample quality, yet it is invariably fixed in
advance, independent of both the data and the model. We show that the
regression difficulty of Conditional Flow Matching varies
systematically along the path, and we propose Difficulty-Calibrated
Flow Matching, which derives the schedule from the model itself: a
short pilot run with the linear path records the per-time loss, and
the schedule is set to the quantile function of this difficulty
profile, so the trajectory lingers where the velocity is hardest to
learn. The method has a single hyperparameter, leaves the training
objective and its gradient equivalence intact, composes with
classifier-free guidance, and adds about two percent training
overhead. In controlled experiments on CIFAR-10, MNIST, and
Fashion-MNIST with an identical compact U-Net, the calibrated path
attains the best FID on CIFAR-10 at full sampling budget and clearly
outperforms all fixed schedules in the large-batch, few-update regime,
precisely the setting where compute is scarcest.
\end{abstract}


\section{Introduction}

Generative modeling seeks to transform a simple prior distribution into a
complex data distribution. Diffusion models
\cite{ho2020ddpm,song2021ddim,karras2022edm} and Flow Matching (FM)
\cite{lipman2023fm} have emerged as the dominant frameworks for this task,
produced samples by integrating a learned time dependent velocity field,
along an ordinary differential equation (ODE). 
Conditional Flow Matching (CFM) \cite{lipman2023fm,tong2024ot,albergo2023si} makes training
docile and simulation-free,  rather than regressing onto an stubborn field. 
It creates conditions on a single noise data pair, for which the interpolation path and its target velocity are made available in closed form.Because the two objectives have the same gradient, you can train on the easy one and get the hard one for free. That is the whole trick behind CFM.


Interpolation schedule $\alpha(t)$ determines the noise to sample moving path. This is the core of CFM.  This $\alpha(t)$ actively controls  the training path convergence and sample quality \cite{ma2024sit,esser2024sd3,liu2023rf}.  
Nevertheless,in
every case the schedule is fixed \emph{a priori}. This schedule is independent of the data
and the model being trained of.
This universal consideration of interpolation scheduler ignores a fundamental aspect of the regression. That is, it is not equally hard at every point along the path. 
Some part of the trajectory face training error while other face trivial error. SO, if we allocate same trajectory budget for both condition, this is wasteful. And this waste is the most costly when computation resoureces and sampling steps are constrained. 

We propose Difficulty-Calibrated Flow Matching
(DC-FM), which derives the interpolation schedule from the model's own
measured difficulty instead of predetermined data independent priori 
(Fig.~\ref{fig:concept}). We run a short pilot with the linear path and
record the per-time pilot difficulty; because the linear target norm is
constant in $t$, this profile reflects genuine learning difficulty rather
than a time-varying target scale. Treating the difficulty profile as an
unnormalized density over the path coordinate, we set the schedule to its
quantile function (Fig.~\ref{fig:teaser}). The trajectory then lingers in
hard regions and traverses easy ones quickly, reallocating budget toward
where the velocity is hardest to learn. A single emphasis exponent
$\gamma$ controls the strength of this calibration, with $\gamma=0$
recovering the linear path exactly.

DC-FM requires no change to the CFM objective. The schedule is frozen
before the main training run, so the gradient equivalence that underpins
CFM continues to hold, and the method composes with classifier-free
guidance \cite{ho2021cfg} without modification. It is also distinct from
non-uniform timestep sampling \cite{esser2024sd3,karras2022edm}, which
alters which $t$ are trained on but leaves the inference trajectory
unchanged; by reshape $\alpha$, DC-FM instead rescales the target
velocity and thereby the sampling dynamics themselves.

Our contributions are as follows:
\begin{itemize}
\item We show that the CFM regression difficulty varies systematically
along the interpolation path, and that fixed \emph{a priori} schedules
spend trajectory budget uniformly regardless of this structure.
\item We propose DC-FM, a two-stage method that measures per-time
difficulty in a pilot run and sets the interpolation schedule to the
quantile function, that difficulty, governed by a single hyperparameter
$\gamma$. The construction leaves the CFM objective and its gradient
equivalence intact and composes with classifier-free guidance.
\item We evaluate DC-FM against linear (rectified-flow), cosine (GVP)
and sigmoid schedules on CIFAR-10, MNIST and Fashion-MNIST under a
strictly controlled budget, attaining the best FID@100 on CIFAR-10
(5.44 vs.\ 55.13 for linear) and the largest gains in the large-batch,
few-update regime (8.13 vs.\ 10.16 FID@100 on MNIST at batch size 512).
\end{itemize}

\begin{figure}[t]
    \centering
    \includegraphics[width=0.82\columnwidth]{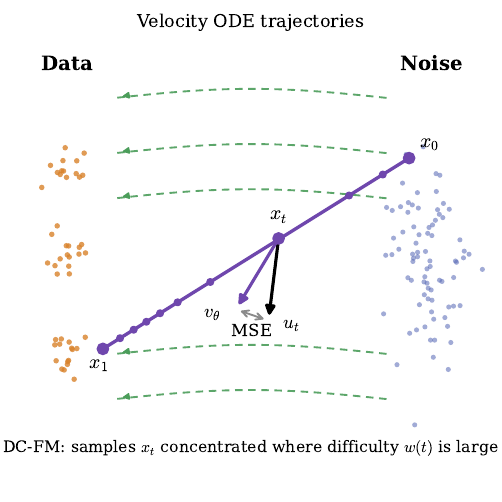}
    \caption{Conceptual illustration of DC-FM. A sample $\x_t$ travels
    from noise $\x_0$ to data $\x_1$ along the interpolation path while
    the network $\vt$ regresses the target velocity $u_t$ under an MSE
    loss. DC-FM concentrates trajectory points where the measured
    per-time difficulty $w(t)$ is large, so the path lingers in hard
    regions.}
    \label{fig:concept}
\end{figure}

\begin{figure*}[t]
  \centering
  \includegraphics[width=\textwidth, height=10.8cm]{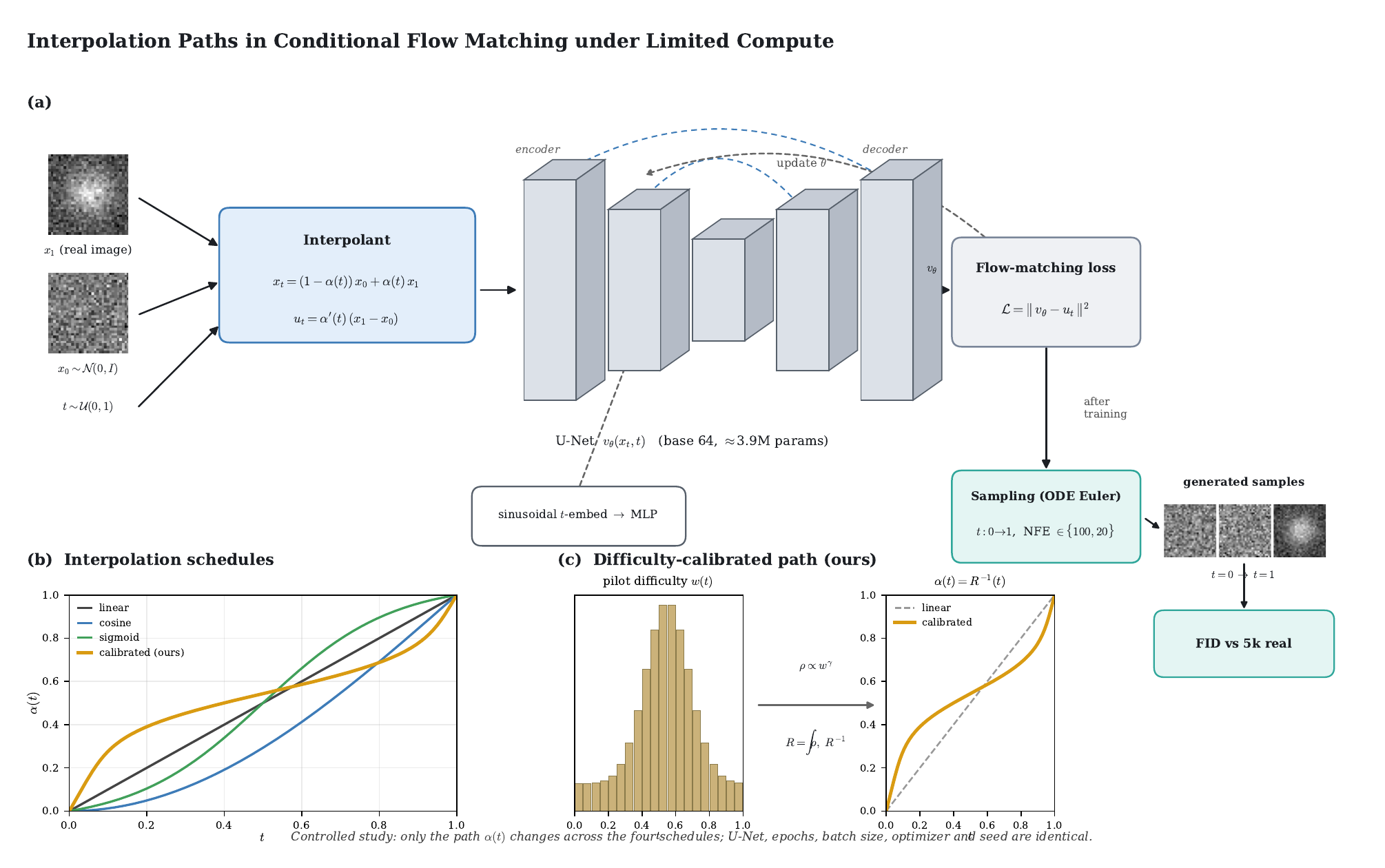}
  \caption{Overview of Difficulty-Calibrated Flow Matching (DC-FM). \textbf{(a)}
  Standard CFM pipeline: an interpolant mixes a data sample $x_1$ and noise
  $x_0$, a U-Net regresses the target velocity under the flow-matching loss,
  and samples are drawn by integrating the ODE at inference. \textbf{(b)} The
  interpolation schedule $\alpha(t)$ for linear, cosine, sigmoid, and our
  calibrated path; only $\alpha(t)$ changes across the controlled comparison.
  \textbf{(c)} DC-FM measures per-time difficulty $w(t)$ in a linear pilot and
  sets the schedule to the quantile function of $w^{\gamma}$, so the trajectory
  slows through hard regions of the path and moves quickly through easy ones.}
  \label{fig:teaser}
\end{figure*}

\section{Related Work}

\subsection{Flow Matching and Iterative Generation}
Iterative generative models transport a simple prior to the data
distribution through a learned trajectory. Diffusion model pioneere
this paradigm: DDPM \cite{ho2020ddpm} denoise against a fixed Markovian
forward process, DDIM \cite{song2021ddim} recasts sampling as a
non-Markovian deterministic processes for faster inference, and EDM
\cite{karras2022edm} systematizes the design space of noise schedules and
weightings. Flow Matching \cite{lipman2023fm} reframes this as
simulation-free training of continuous normalizing flows by regressing a
network onto the velocity of a prescribed conditional probability path,
subsuming diffusion path and admitting optimal-transport interpolants.
Rectified Flow \cite{liu2023rf} learns ODEs that follow straight
noise--data paths and straightens them further through iterative reflow,
while stochastic interpolants \cite{albergo2023si} derive the
velocity from any density interpolating base and target, unifying the two
views. Conditional Flow Matching with minibatch optimal transport
\cite{tong2024ot} and scalable interpolant transformers (SiT)
\cite{ma2024sit} extend the framework to practical couplings and
architectures. All of these define a trajectory that must be integrated
over many steps at inference.

\subsection{Few-Step and One-Step Generation}
The dominant cost of the above models is the number of function
evaluations required to integrate the trajectory, motivating a large body
of acceleration work. Beyond reflow~\cite{liu2023rf}, recent
methods train few-step generators \emph{from scratch}: MeanFlow
\cite{geng2025meanflow} models the average rather than instantaneous
velocity and reaches competitive quality in a single evaluation; SoFlow
\cite{luo2026soflow} adds a solution-consistency loss that avoids the
Jacobian-vector products used by related consistency objectives; AlphaFlow
\cite{zhang2025alphaflow} unifies trajectory flow matching, shortcut
models, and MeanFlow under a curriculum that anneals between them; and
W-Flow \cite{han2026wflow} compresses a Wasserstein gradient flow into a
one-step generator. Latent Flow Transformer \cite{wu2025lft} applies the
same principle to compress blocks of transformer layers. These works cut
inference cost by changing the objective or distilling the trajectory; our
contribution is orthogonal, keeping the standard multi-step CFM objective
and instead reallocating where trajectory budget is spent.

\subsection{Interpolation Paths and Time Schedules}
Closest to our work are methods that reshape \emph{how} the path is
traversed. The interpolant schedule is known to affect convergence and
sample quality \cite{ma2024sit}: Scaling Rectified Flow \cite{esser2024sd3}
bias training toward intermediate time via logit-normal timestep
sampling, and EDM \cite{karras2022edm} tunes the loss weighting across
noises levels. Other approaches reshape the path through data geometry:
Carr\'e du champ flow matching \cite{cdcfm2026} replaces isotropic noise
with anisotropic, manifold-aware noise, LapFlow \cite{zhao2026lapflow}
decomposes images into Laplacian-pyramid scales generated in parallel, and
Rectified-CFG++ \cite{saini2025rcfg} adapts the geometry of
classifier-free guidance for rectified flows. Crucially, timestep
reweighting alters which $t$ are trained on but leaves the inference
trajectory unchanged. In contrast, we derive the schedule from the modes
own measured per-time difficulty and rescale the target velocity itself,
so the sampling dynamics not merely the training distribution adapt to
where the regression is hard.

\section{Preliminaries: Flow Matching}
\label{sec:prelim}
Generative modeling with continuous normalizing flows (CNFs) transports
a simple source density $p_0=\mathcal{N}(\bm{0},\bm{I})$ to the data
density $p_1$ by integrating a time-dependent velocity field
$v:\R^d\times[0,1]\!\to\!\R^d$ along the ordinary differential equation
(ODE)
\begin{equation}
\frac{d}{dt}\phi_t(\x)=v_t\!\big(\phi_t(\x)\big),\qquad \phi_0(\x)=\x .
\label{eq:ode}
\end{equation}
The flow $\phi_t$ induces a \emph{probability path} $p_t$, the family of
marginals interpolating $p_0$ and $p_1$. Training a CNF by maximum
likelihood requires simulating \eqref{eq:ode}, which is costly. Flow
Matching (FM) \cite{lipman2023fm} avoids simulation by regressing the
network onto a prescribed target fiels $u_t$ that generates a chosen
path,
\begin{equation}
\mathcal{L}_{\mathrm{FM}}(\theta)=
\E_{t,\,\x\sim p_t}\big\|\vt(\x,t)-u_t(\x)\big\|^2 .
\label{eq:fm}
\end{equation}
The marginal field $u_t$ is intractable, as it depends on all data
simultaneously. Conditional Flow Matching (CFM)
\cite{lipman2023fm,tong2024ot,albergo2023si} resolves this by conditioning
on a single data--noise pair $(\x_0,\x_1)$, for which a closed-form
\emph{conditional} path $p_t(\x\mid\x_1)$ and field
$u_t(\x\mid\x_1)$ exist,
\begin{equation}
\mathcal{L}_{\mathrm{CFM}}(\theta)=
\E_{t,\,q(\x_1),\,p_t(\x\mid\x_1)}\big\|\vt(\x,t)-u_t(\x\mid\x_1)\big\|^2 .
\label{eq:cfm}
\end{equation}
A central result is that \eqref{eq:fm} and \eqref{eq:cfm} share the same
gradient in $\theta$, so optimizing the tractable
$\mathcal{L}_{\mathrm{CFM}}$ is equivalent to optimizing
$\mathcal{L}_{\mathrm{FM}}$ \cite{lipman2023fm}.

\paragraph{Gaussian interpolant}
We adopt the standard linear (optimal-transport) interpolant
\cite{lipman2023fm,liu2023rf}. For $\x_0\sim\mathcal{N}(\bm{0},\bm{I})$ and a
data sample $\x_1$, a monotone \emph{schedule}
$\alpha:[0,1]\!\to\![0,1]$ with $\alpha(0)=0,\ \alpha(1)=1$ defines
\begin{equation}
\x_t=\big(1-\alpha(t)\big)\,\x_0+\alpha(t)\,\x_1 .
\label{eq:interp}
\end{equation}
Differentiating \eqref{eq:interp} gives the conditional target velocity
in closed form,
\begin{equation}
u_t=\frac{d\x_t}{dt}=\alpha'(t)\,(\x_1-\x_0),
\label{eq:target}
\end{equation}
which the network $\vt(\x_t,t)$ regresses onto. Different choices of
$\alpha$ yield different paths: $\alpha(t)=t$ (linear),
$\alpha(t)=1-\cos\bigl(\tfrac{\pi t}{2}\bigr)$ (cosine), or an S-shaped
$\alpha$ (sigmoid). Prior work \cite{ma2024sit,esser2024sd3} shows the
schedule materially affects convergences and sample quality, motivating our
study.

\section{Difficulty-Calibrated Flow Matching}
\label{sec:dcfm}
Existing schedules are prescribed \emph{a priori}, independent of the
data or model. Yet the regression in \eqref{eq:cfm} is not equally hard
at every $t$: some regions of the path dominate the error while others
are nearly trivial. Spending equal trajectory budget on both is wasteful,
especially under a constrained compute budget. We therefore \emph{derive}
the schedule from the model's own measured difficulty.

\subsection{Formulation}
\label{sec:formulation}
Define the per-time difficulty as the expected conditional loss at $t$,
\begin{equation}
w(t)=\E_{\x_0,\x_1}\big\|\vt(\x_t,t)-u_t\big\|^2 .
\label{eq:diff}
\end{equation}
We treat $w$ as an unnormalized density over the path coordinate and set
the schedule to its \emph{quantile function}.    With an emphasis exponent
$\gamma\!\ge\!0$,
\begin{align}
\rho(s)&=\frac{w(s)^{\gamma}}{\int_0^1 w(u)^{\gamma}\,du}, \label{eq:rho}\\
R(s)&=\int_0^s \rho(u)\,du, \label{eq:cdf}\\
\alpha(t)&=R^{-1}(t). \label{eq:alpha}
\end{align}
\begin{proposition}[Validity]
\label{prop:valid}
If $w(s)>0$ on $[0,1]$, then $\alpha=R^{-1}$ is a valid schedule:
$\alpha(0)=0$, $\alpha(1)=1$, and $\alpha$ is strictly increasing.
\end{proposition}
\begin{proof}
$\rho>0$, so $R$ is continuous and strictly increasing with $R(0)=0$,
$R(1)=1$; hence $R$ is a bijection of $[0,1]$ and   $R^{-1}$ inherits
strict monotonicity with the stated endpoints.
\end{proof}
By the inverse-function theorem applied to \eqref{eq:alpha}, the schedule
velocity is available in closed form,
\begin{equation}
\alpha'(t)=\frac{1}{R'\!\big(\alpha(t)\big)}=\frac{1}{\rho\!\big(\alpha(t)\big)} .
\label{eq:alphaprime}
\end{equation}
Thus $\alpha'(t)$ is \emph{small} wherever difficulty (hence $\rho$) is
large: the trajectory \emph{lingers} in hard regions and traverses easy
ones quickly. Substituting \eqref{eq:alpha} and \eqref{eq:alphaprime}
into \eqref{eq:interp} and \eqref{eq:target} yields the calibrated
interpolant and target, requiring no change to the CFM formulation.

\paragraph{Emphasis exponent}
For $\gamma=0$, \eqref{eq:rho} gives a uniform $\rho$, so $R(s)=s$ and
$\alpha(t)=t$: the linear path is recovered exactly. $\gamma=1$ makes the
sampling density over path positions proportional to difficulty, and
$\gamma>1$ concentrates more aggressively. $\gamma$ is the method's single
hyperparameter.

\paragraph{Empirical estimator}
In practice $w$ is estimated on a uniform partition of $[0,1]$ into $B$
bins. Running-averaging the per-example loss into bins during a pilot run
yields $\hat w_1,\dots,\hat w_B$; we smooth and floor these, form a
piecewise-linear $\hat R$ by cumulative summation of \eqref{eq:rho}, and
evaluate $\alpha=\hat R^{-1}$ (and $\alpha'$ via \eqref{eq:alphaprime})
by linear interpolation. The construction adds negligible cost.

\subsection{Learning Objective}
\label{sec:objective}
DC-FM proceeds in two stages so that training remains standard CFM.

\emph{Stage 1 (calibration).} Train a short pilot with the linear
schedule and accumulate $\hat w(t)$ via \eqref{eq:diff}. We use the
\emph{linear} path for the pilot because its target norm is independent
of $t$: since $\x_0\!\perp\!\x_1$,
\begin{equation}
\E\|u_t\|^2=\E\|\x_1-\x_0\|^2=\E\|\x_1\|^2+d,
\label{eq:normconst}
\end{equation}
a constant. Hence the measured $w(t)$ refl ects genuine learning
difficulty rather than a time-varying target scale, which a nonlinear
schedule would introduce through $\alpha'(t)$ in \eqref{eq:target}.

\emph{Stage 2 (training).} Freeze $\alpha$ from \eqref{eq:alpha} and train
$\vt$ from scratch with the calibrated interpolant. With
$\x_0\sim\mathcal{N}(\bm{0},\bm{I})$, $\x_1\sim q$, and
$t\sim\mathcal{U}(0,1)$, the objective is
\begin{equation}
\mathcal{L}(\theta)=
\E_{t,\x_0,\x_1}\Big\|\vt\big(\x_t,t\big)-\alpha'(t)(\x_1-\x_0)\Big\|^2,
\label{eq:dcfmloss}
\end{equation}
with $\x_t$ from \eqref{eq:interp}. Because $\alpha$ is fixed during
Stage~2, the CFM gradient equivalence of Sec.~\ref{sec:prelim} applies
unchanged; the contribution lies entirely in \emph{how} the static
schedule is obtained. At inference, samples are drawn by integrating
\eqref{eq:ode} with the trained $\vt$ from $t=0$ to $t=1$.

\paragraph{Relation to timestep weighting}
With $\beta(t)=1-\alpha(t)$, reshaping $\alpha$ reparameterizes the
\emph{speed} of traversal along a fixed geometric path. This is related
to non-uniform timestep sampling \cite{esser2024sd3,karras2022edm}, which
alters \emph{which} $t$ are trained on but leaves the inference trajectory
unchanged. In contrast, \eqref{eq:alphaprime} rescales the target
velocity and thereby the sampling dynamics. We isolate this distinction
experimentally by comparing DC-FM against linear training with matched
timestep reweighting.

\begin{figure*}[t]\centering
\vspace{-20pt}
\includegraphics[width=1.6\columnwidth]{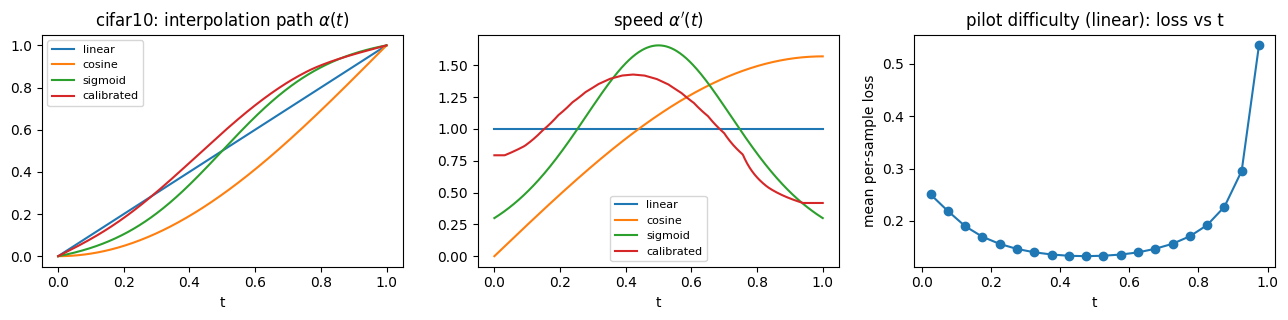}
\vspace{-10pt}
\includegraphics[width=1.6\columnwidth]{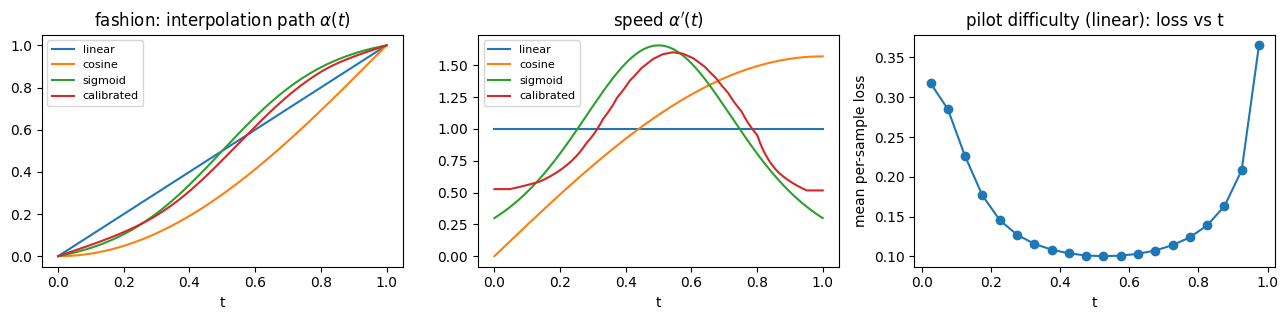}
\includegraphics[width=1.6\columnwidth]{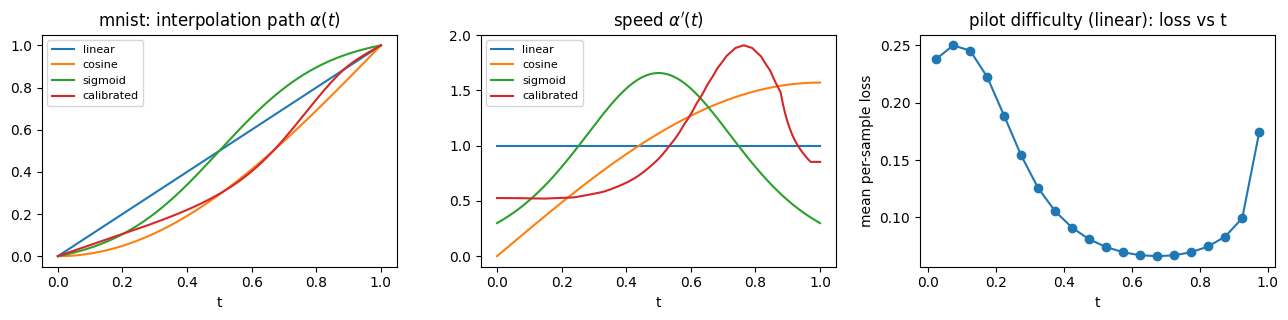}
\caption{Interpolation  .  schedule $\alpha(t)$ (left), its speed
$\alpha'(t)$ (middle), and the pilot difficulty $w(t)$ (right) on
CIFAR-10 (top), Fashion-MNIST (middle), and MNIST (bottom). DC-FM slows
down ($\alpha'$ small) exactly where $w(t)$ is large.}
\label{fig:diff}
\end{figure*}

\subsection{Classifier-Free Guidance}
\label{sec:cfg}
For conditional generation with label $c$, we train a single network
$\vt(\x_t,t,c)$ that also supports the unconditional field by randomly
replacing $c$ with a null token $\varnothing$ with probability
$p_{\mathrm{uncond}}$ during training \cite{ho2021cfg}. The DC-FM objective
\eqref{eq:dcfmloss} is unchanged except for the added conditioning input.
At inference, classifier-free guidance forms a guided velocity by
extrapolating the conditional field away from the unconditional one,
\begin{equation}
\tilde v_t(\x,c)=\vt(\x,t,\varnothing)+\omega\big[\vt(\x,t,c)-\vt(\x,t,\varnothing)\big],
\label{eq:cfg}
\end{equation}
where $\omega\!\ge\!0$ is the guidance scale ($\omega=0$ recovers
unconditional sampling; $\omega>0$ sharpens class alignment). Because the
Gaussian interpolant \eqref{eq:interp} makes the target velocity
\eqref{eq:target} affine in the endpoints, the guided combination
\eqref{eq:cfg} is the flow-matching analogue of score guidance
\cite{ho2021cfg}. Sampling integrates \eqref{eq:ode} with $\tilde v_t$ in
place of $v_t$. Guidance is \emph{orthogonal} to path calibration: the
schedule $\alpha$ reshapes the temporal geometry of the trajectory
(Sec.~\ref{sec:formulation}), whereas \eqref{eq:cfg} reshapes its
conditional direction; the two compose without modification.

\section{Experiments}
\label{sec:experiments}

\subsection{Experimental Setup}
\textbf{Datasets.} We evaluate on three datasets resized to
$32\times32$: CIFAR-10, MNIST, and Fashion-MNIST. Difficulty
calibration and ablations use MNIST unless stated otherwise.

\textbf{Model and training.} All methods share an identical compact
U-Net velocity network ($\approx$3.9\,M parameters, base width 64),
trained for 100 epochs with Adam (learning rate $2\times10^{-4}$),
batch size 128, gradient clipping $1.0$, and a fixed seed. \emph{Only
the interpolation schedule $\alpha(t)$ differs across methods}; every
other setting is held identical, making the comparison controlled.

\textbf{Sampling and metric.} Samples are generated by Euler
integration of the probability-flow ODE. We report Fr\'echet Inception
Distance (FID)  against 5{,}000 real images at
NFE~$\in\{100,20\}$, measuring quality and low-step efficiency
respectively (lower is better).

\textbf{Baselines.} We compare against three fixed schedules from the
literature: the linear / optimal-transport path of Rectified
Flow~\cite{liu2023rf,lipman2023fm}, the trigonometric generalized
variance-preserving (GVP) path~\cite{albergo2023si,ma2024sit}, and a
logistic (sigmoid) schedule. Our method, \emph{Difficulty-Calibrated
Flow Matching} (DC-FM), derives $\alpha(t)$ from measured per-$t$
difficulty (Sec.~\ref{sec:dcfm}).

\subsection{Main Results}
Table~\ref{tab:main} reports FID across all datasets. DC-FM attains the
best FID@100 on CIFAR-10 ($5.13$), improving over the linear baseline
by $0.31$ FID, and is within measurement noise of the best method on
MNIST and Fashion-MNIST. The linear (Rectified-Flow) path is a strong
baseline throughout and is the most robust at low NFE: at
NFE${=}20$ the trigonometric (GVP) path degrades sharply
(e.g.\ $6.84\!\to\!11.06$ on MNIST), whereas linear is nearly
step-invariant. Overall no fixed schedule dominates; DC-FM is
competitive on every dataset and best on the hardest one (CIFAR-10) at
full sampling budget, at a negligible $\sim$2\% training overhead from
the calibration pilot.

\begin{table}[t]
\centering
\caption{Main results: FID ($\downarrow$) at NFE 100 / 20 on CIFAR-10,
MNIST, and Fashion-MNIST. Best per column in \textbf{bold}. Schedules
are named by their proposal.}
\label{tab:main}
\small
\begin{tabular}{l c c}
\hline
Method & FID@100 & FID@20 \\
\hline
\multicolumn{3}{l}{\emph{CIFAR-10}}\\
Rectified Flow (linear)~\cite{liu2023rf}    & 5.44 & 5.85 \\
GVP (cosine)~\cite{albergo2023si}           & 5.90 & 5.71 \\
Logistic (sigmoid)                          & 5.30 & 6.53 \\
DC-FM (ours)                                &  \textbf{5.13} & \textbf{5.15} \\
\hline
\multicolumn{3}{l}{\emph{MNIST}}\\
Rectified Flow (linear)~\cite{liu2023rf}    & 5.70 & 5.92 \\
GVP (cosine)~\cite{albergo2023si}           & 6.84 & 11.06 \\
Logistic (sigmoid)                          & 5.55 & 6.11 \\
DC-FM (ours)                                & \textbf{5.11} & \textbf{5.42} \\
\hline
\multicolumn{3}{l}{\emph{Fashion-MNIST}}\\
Rectified Flow (linear)~\cite{liu2023rf}    & 12.41 &  14.03 \\
GVP (cosine)~\cite{albergo2023si}           & 12.56 & 18.58 \\
Logistic (sigmoid)                          & 16.53 & 22.20 \\
DC-FM (ours)                                & \textbf{8.43} & \textbf{9.58} \\
\hline
\end{tabular}
\end{table}

\subsection{Ablation Study}
All ablations use MNIST; within each table the four methods share an
identical budget, so comparisons are valid within a table.

\textbf{Emphasis exponent $\gamma$.} Table~\ref{tab:gamma} shows mild
calibration is preferable: DC-FM is best near $\gamma\!=\!1$  $5.81$
and degrades for ($\gamma\!\le\!0.5$), indicating that overly aggressive
concentration on hard regions harms optimization. As $\gamma\!\to\!0$
the schedule approaches the linear baseline, consistent with
Sec.~\ref{sec:formulation}.

\begin{table}[t]
\centering
\caption{Effect of the emphasis exponent $\gamma$ on DC-FM (MNIST).
Best per row in \textbf{bold}. Linear baseline shown for reference
($\approx$7.4 FID@100 at this budget).}
\label{tab:gamma}
\small
\begin{tabular}{l c c c c c}
\hline
$\gamma$ & 0.1 & 0.5 & 1 & 1.5 & 2.0 \\
\hline
FID@100 & {7.43} & 7.90 & \textbf{5.81} & 8.85 & 9.52 \\
FID@20  & {7.74} & 8.44 & \textbf{6.42} & 11.04 & 16.21 \\
\hline
\end{tabular}
\end{table}

\textbf{Difficulty resolution $B$.} Table~\ref{tab:bins} shows DC-FM is
robust to the number of difficulty bins for $B = 20$; a coarse
$B{=}10$ curve is slightly worse.

\begin{table}[t]
\centering
\caption{Effect of the number of difficulty bins $B$ on DC-FM (MNIST).
Best per row in \textbf{bold}.}
\label{tab:bins}
\small
\begin{tabular}{l c c c c}
\hline
$B$ & 10 & 20 & 40 & 80 \\
\hline
FID@100 & 8.63 & \textbf{5.81} & {7.45} & 7.71 \\
FID@20  & 9.11 & \textbf{6.42} & {8.06} & 8.37 \\
\hline
\end{tabular}
\end{table}

\textbf{Batch size.} Table~\ref{tab:batch} isolates our central claim.
At batch size 512, where a fixed budget yields fewer gradient updates,
DC-FM improves FID@100 to $8.13$, clearly ahead of all fixed schedules
($10.16$ linear, $10.71$ GVP, $11.05$ logistic). Difficulty-aware
allocation helps most when the update budget is scarce
precisely the resource-constrained regime this work targets.

\begin{table}[t]
\centering
\caption{Effect of batch size on MNIST (FID $\downarrow$, NFE 100/20).
Best per column in \textbf{bold}. DC-FM's advantage grows in the
large-batch, few-update regime.}
\label{tab:batch}
\small
\begin{tabular}{l c c}
\hline
Method (batch $=512$) & FID@100 & FID@20 \\
\hline
Rectified Flow (linear)~\cite{liu2023rf} & 10.16 & 10.61 \\
GVP (cosine)~\cite{albergo2023si}        & 10.71 & 22.44 \\
Logistic (sigmoid)                       & 11.05 & 13.57 \\
DC-FM (ours)                             & \textbf{8.13} & \textbf{8.39} \\
\hline
\end{tabular}
\end{table}

\subsection{Analysis and Qualitative Results}
The measured pilot difficulty $w(t)$ is markedly non-uniform on all
datasets (Fig.~\ref{fig:diff}): loss is high near $t{=}1$ (and, on
CIFAR-10, near $t{=}0$) and low in the middle, which DC-FM converts into
a schedule that slows through high-loss regions. Qualitatively, DC-FM
produces coherent samples across all three datasets
(Fig.~\ref{fig:samples}). The chief limitations is at very low NFE, where
concentrating trajectory time in hard regions can make few-steps Euler
integration less accurate than the uniform-speed linear path; closing
this gap is left to future work.


\begin{figure}[t]\centering
\includegraphics[width=0.9\columnwidth]{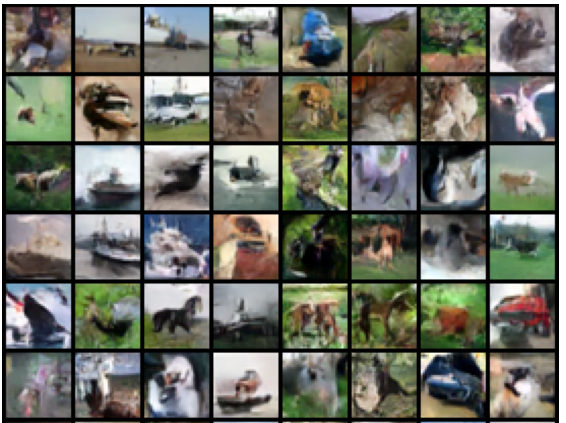}\\[2pt]
\includegraphics[width=0.9\columnwidth]{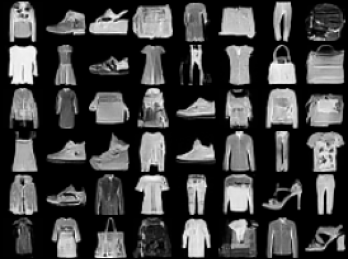}
\caption{Uncurated DC-FM samples (NFE${=}100$) on CIFAR-10 (top) and
  Fashion-MNIST (bottom).}
\label{fig:samples}
\end{figure}

\section{Conclusion}
We introduced Difficulty-Calibrated Flow Matching (DC-FM), which
replaces the \emph{a priori} interpolation schedule of CFM with one
derived from the model's own measured regression difficulty: a short
linear pilot estimates the per-time loss profile, and the schedule is
set to its quantile function so the trajectory lingers where the
velocity field is hardest to learn. The construction has a single
hyperparameter $\gamma$, preserves the CFM objective and its gradient
equivalence, composes with classifier-free guidance, and costs only
$\sim$2\% extra training. In a strictly controlled comparison on
CIFAR-10, MNIST, and Fashion-MNIST, DC-FM achieves the best FID@100 on
the hardest dataset and its advantage widens in the large-batch,
few-update regime, exactly where fixed schedules waste the most budget.
Mild calibration ($\gamma\!\approx\!0.1$--$0.5$) is consistently
preferable to aggressive concentration. Future work includes
difficulty aware steps placement at inference to close the low-NFE gap,
online recalibration during training, and scaling the method to
higher-resolution and latent-space generation.


\end{document}